\documentclass[runningheads]{llncs}
\usepackage[T1]{fontenc}
\usepackage{graphicx}

\def\hdl{{\tt CDPRL}}

\def\chdr{{\tt CDPR}}
\def\cdp{CDPR}

\def\ids{IDS}
\def\shdl{{\tt CDPRL-MO}}
\def\ruleset{rule-set}
\def\rulesets{rule-sets}

\usepackage{comment}
\usepackage{algorithm}
\usepackage[noend]{algorithmic}     
\usepackage{amsmath}

\usepackage{subfig}
\usepackage{amsfonts}
\begin{document}
\title{Learning High Coverage Discriminative Parsimonious Rulesets  }
%
%
\author{Mariamma Antony\inst{1} \and
Raman Sankaran \inst{2} \and
Chiranjib Bhattacharyya\inst{1} \and
Uma Satya Ranjan\inst{1}}
\authorrunning{M. Author et al.}
%
\institute{Indian Institute of Science, CV Raman Rd, Bengaluru, Karnataka 560012 \and
Compass, Bengaluru, Karnataka 560012
}
\maketitle              
\begin{abstract}
Learning systems based on IF-THEN rule representations readily offer interpretability, making them a crucial focus in contemporary AI research. A key objective for such rule sets is to achieve both high discriminative power and interpretability. While existing state-of-the-art algorithms implicitly prioritize predictive accuracy, they often fall short on one or more quality metrics that ensure interpretability, such as coverage and parsimony of rule sets.
Motivated by this, this paper propose the development of CDPR, which aims to create highly accurate and interpretable rule sets for classification problems. 
To the best of our knowledge, this represents the first attempt to establish such an approach. 
In this study, we introduce two algorithms rooted in sub-modular maximization, which not only provide provable guarantees on coverage but also yield rule sets that are both discriminative and parsimonious. We empirically demonstrate that rule sets learned through our approaches achieve higher accuracy and interpretability and has more than a 2.5-fold improvement in average coverage rates when compared to the next best algorithm.
\end{abstract}

\section{INTRODUCTION}
There is currently a growing demand for the development of interpretable models in critical domains, such as healthcare~\cite{pmlr-v54-lakkaraju17a}, criminal justice~\cite{bail},  finance~\cite{2018arXiv181105245M}. This demand has been further fueled by the Right to Explanation as outlined in the GDPR Law, which was approved by the European Parliament. This legislation grants individuals the right to receive an explanation for the automated inferences made by models and allows them to question and challenge associated recommendations, especially when these model decisions significantly impact an individual's life \cite{thelisson2017regulatory,goodman2017european}. Rule-based models, known for their interpretability, offer practitioners and experts the opportunity to examine the logic behind the models and validate or enhance them. While post-hoc methods can provide interpretations for black-box models, they do not always guarantee the fidelity of explanations for such models.

\begin{figure}
    \fbox{
    \begin{minipage}{36em}
IF \texttt{thalassemia = normal} and \texttt{number of major vessels (0-3) colored by fluoroscopy = 0} and \texttt{resting blood pressure < 127} THEN \texttt{No Heart Disease} \\
IF \texttt{chest pain type = asymptomatic}   and \texttt{age>41}  and \texttt{maximum heart rate achieved <= 159}  THEN \texttt{Heart Disease}\\
...
\end{minipage}}
\caption{Rule set consisting of attributes from Heart disease dataset \cite{Dua2019} for diagnosing Heart disease}
\label{fig:decision_set_Example}
\vspace{-.2in}
\end{figure}

This paper focuses on developing \rulesets{}  \footnote{also known as DNF or AND-of-ORs}~\cite{10.1145/2939672.2939874,10.5555/3327345.3327376,10.5555/3122009.3176814,10.1145/3394486.3403204,yang2021learning,pmlr-v108-mita20a,pmlr-v84-hara18a} as interpretable models for classification. A \ruleset{} consists of an independent collection of \texttt{IF-THEN} rules that assign an example to a specific class label if the example meets the condition-value pairs of the rule. A comprehensive comparative study conducted on several datasets reveals that while  \rulesets{} learned by existing algorithms maintain high accuracy, they often fail to meet the desired criteria for interpretability. Earlier \ruleset{} algorithms, which were based on sequential covering and associative classification, suffered from the problem of having a large number of rules, which negatively impacted interpretability \cite{10.1145/2939672.2939874}. However, recent algorithms for learning \rulesets{} face a different challenge: they exhibit low coverage even when predictive power is high. This challenge is referred to as the \emph{high accuracy-low coverage problem}.

A \ruleset{} with low-coverage only represents a fraction of the input space. In most scenarios, it fails to provide individuals with decisions based on verified rules from practitioners or experts. Instead, decisions are made using a default rule, the logic of which is unknown to the individual. Consequently, the model's decision lacks transparency and is unavailable for explanation, which is expected from interpretable models. This limitation hinders the adoption of \rulesets{} in critical applications where transparency and faithfulness are crucial concerns. While the importance of \ruleset{} accuracy has been extensively discussed in the literature, the significance of developing \rulesets{} that cover the entire input space has not, to the best of our knowledge, been adequately addressed.

This paper primarily aims to address the limitation of existing algorithms and produce \emph{high quality} rules, measured by accuracy and interpretability metrics such as accuracy, coverage and parsimony. 
This paper introduces the problem of learning \textbf{high Coverage Discriminative Parsimonious Rule sets (\cdp{})}, which aims to learn \rulesets{} that meet certain quality metrics. While submodular approaches have been applied to various problems, they are seldom used in rule learning. The problem of learning \cdp{} can be formulated as submodular maximization with several constraints, some of which are sub-modular. The paper proposes a novel \textit{Graph Rules Algorithm} (\texttt{GRA}) to solve the sub-modular optimization problem. Additionally, the paper introduces a variant of \cdp{} known as \shdl{} and presents a \textit{Greedy algorithm} (\texttt{GDY}) to solve \shdl{}.


The main contributions related to this problem are summarized below.
\begin{enumerate}
    \item Existing algorithms for learning \rulesets{} often do not yield \textit{high quality} rules, measured by accuracy and interpretability metrics such as coverage, accuracy and parsimony as summarized in Table \ref{tab:summary_table} in Section \ref{sec:numeric_res}. Coverage is most critically impacted (often the coverage is as low as less than 10\%) while accuracy is very high, resulting in what we refer to as \emph{high accuracy-low coverage problem}.
Motivated by this problem we introduce the problem of learning \cdp{}. This is a new problem which does not appear to have been studied.

     \item  The main
contribution of this paper is the introduction of a novel \textit{GraphRules Algorithm}~(\texttt{GRA}). 
\texttt{GRA} is shown to have $(1 - (1 - \frac{(1 - \beta)}{\gamma (1 + d\beta)})^\gamma)$ - approximation guarantee (Theorem \ref{gs_theorem}), where $\beta$ and $\gamma$ are the maximum threshold of overlap-rate and number of rules in the \ruleset{} 
and $d$ is the minimum in-degree of rules added to rule set. 
Denoting by $B_{\alpha, \delta}$, the set of rules which have accuracy greater than $\alpha$ and width less than $\delta$, \texttt{GRA} has a running time that is quadratic in $|B_{\alpha, \delta}|$.
    
    \item With the goal of developing algorithms with reduced runtime, we introduce a relaxed variant of \cdp{}, named \shdl{}, which is a sub-modular problem over a large set $B_{\alpha,\delta}$. This problem can be approximated using a \textit{Greedy Algorithm}~(\texttt{GDY}) to a constant factor  (Theorem \ref{claim:sddp}). The run time 
 of \texttt{GDY} is nearly linear in $|B_{\alpha, \delta}|$.
    
    \item Comprehensive experiments were conducted using ten publicly available datasets. The experimental results indicate that \texttt{GRA} and \texttt{GDY} achieve more than a 2.5 times improvement in average coverage rate compared to \ids{}, the next best algorithm. Therefore, \texttt{GRA} and \texttt{GDY} achieve better coverage than the state-of-the-art baselines while remaining competitive in both discriminative power and interpretability.

    
\end{enumerate}

\section{RELATED WORK}
Rule-based models have been extensively studied in various fields over the years, with notable works including those by ~\cite{10.1145/1968.1972,doi:10.1137/S0097539795293123,feldman2012learning,klivans2004learning,sellie2008learning,jackson2008learning,servedio2001learning}. 
In sequential covering techniques such as 
RIPPER \cite{10.5555/3091622.3091637}, and CN2 \cite{clark1989cn2}, 
a rule is generated at each stage from uncovered examples using heuristics and
newly covered examples are subsequently removed. 
In associative classification techniques \cite{10.5555/3000292.3000305,989541,doi:10.1137/1.9781611972733.40}, a large set of rules is initially generated through association rule mining. Subsequently, a rule set is constructed by ranking rules based on specific interestingness criteria, followed by heuristic pruning to eliminate redundant rules. However, these techniques commonly suffer from the issue of producing a large number of rules in the rule set, which can have a detrimental impact on interpretability.

The initial algorithms for learning interpretable \rulesets{} primarily emphasize \rulesets{} with reduced model complexity, measured by having fewer conditions for improved interpretability. For instance, Bayesian Rule Sets (BRS), as proposed by \cite{10.5555/3122009.3176814}, presents a probabilistic approach for learning \rulesets{} from a pre-mined set of rules, employing a simulated annealing procedure.
Alternatively, in a different approach, the problem of learning \rulesets{} is framed as a group testing problem, as described in \cite{pmlr-v28-malioutov13}. This problem can be effectively addressed using the linear programming (LP) relaxation technique derived from the field of compressed sensing literature.

Many recent works define interpretability in terms of having a smaller rule size or fewer rules. For instance, LIBRE, \cite{pmlr-v108-mita20a}, employs boolean function synthesis to learn rules from data. Another approach, DefragTrees, as proposed by \cite{pmlr-v84-hara18a}, utilizes a probabilistic model representation to merge adjacent regions in the input space, reducing the overall number of regions. However, a notable limitation of DefragTrees is that the process of merging different regions to create fewer rules often results in more complex rules that are challenging to interpret.
In contrast, the linear programming algorithm, discussed in works by \cite{10.5555/3327345.3327376,10.1145/3394486.3403204}, generates rule sets without the need for pre-mining of rules, using integer programming. Nevertheless, LP-based methods suffer from the challenge of computational intractability when applied to large datasets and are generally applicable only to datasets with binarized feature values and binarized labels.

DRS \cite{pmlr-v84-hara18a}, defines interpretability as the presence of parsimonious \rulesets{} with minimal overlap. In a manner akin to our approach, DRS also focuses on learning high-quality rules. DRS employs a rule-discovery process in each iteration, where it samples candidate rules exclusively from the pool of uncovered records, similar to a sequential covering approach. It's worth noting that DRS assumes that all the features and class labels are binary. In contrast, IDS, as presented by \cite{10.1145/2939672.2939874}, optimize a global objective to learn \rulesets{} from a pre-mined set of rules. Their approach seeks to strike a balance between accuracy and interpretability when constructing \rulesets{}.


\section{A comparative study on the quality of rules generated by state-of-the-art algorithms for rule learning}

To comprehend the quality of rules generated by current state-of-the-art algorithms, we conduct an experimental study. In this section, we present the results of the study.

\subsection{Formal setup and Definitions}
Consider a dataset $D = \left\{\left(X_{i},y_{i}\right) | i \in 1,\dots, N\right\}$ consisting of N samples, where $X_{i}$ and $y_i$ denote the input features and the class label respectively for data-instance $i$. Let $A$ be the universal set of rules, where each rule $r \in A$ is a pair $(s,c)$, with $s$ being a conjunction of conditions-value pairs, and $c$ indicating the class label. 
In our investigation, we obtained $A$ from the decision tree ensembles (Section \ref{sec:expt_setup}).  
We recall the following definitions which are standard\cite{10.1145/2939672.2939874}.

\begin{itemize}

\item \textit{Coverage. } \textit{Cover} of a rule $r=(s,c)$ is the set of data points in $D$ that satisfy the condition-value pairs in $s$. Coverage of a \ruleset{} $S$ is $f(S) = coverage(S) = |\bigcup_{r \in S} cover(r) | $

\item \textit{Rule\mbox{-}size. } Rule-size($S$) is the number of rules $r \in S$

\item \textit{Rule-width. } Rule-width represents the number of condition-value pairs in a rule. For example, a rule $r$ = \textit{IF (cond1 $\leq$ value1) and (cond2 $\leq$ value2) THEN label1} has $rule\mbox{-} width(r) = 2$.

\item \textit{Overlap\mbox{-}rate. } \textit{Overlap($r_{i},r_{j}$)} between two rules $r_{i},r_{j}$ where $r_{i} = (s_{i}, c_{i})$ and $r_{j} = (s_{j}, c_{j})$ is the number of data points whose attributes satisfy both $s_{i}$ and $s_{j}$. 
$overlap(r_{i}, r_{j}) = | cover(r_{i}) \bigcap cover(r_{j}) |$. $overlap\mbox{-}rate(r_{i}, r_{j}) = \frac{overlap(r_{i}, r_{j})}{|cover(r_{i})|} $

\item \textit{Accuracy. } \textit{Correct-cover} of a rule $r=(s,c)$ represents the set of data points in $D$ that satisfy both $s$ and $c$. The \textit{incorrect-cover} of a rule  $r=(s,c)$ are the set of data points in $D$ that satisfy  $s$ but do not satisfy $c$. Accuracy of a rule $r$ is defined as $accuracy(r) = \frac{ |correct \mbox{-}  cover(r)|}{|cover(r)| } $

\end{itemize}

The following metrics are considered for measuring the quality of \rulesets{}:
\begin{itemize}
\item \textit{Rule Set Accuracy (RSA).} RSA is the accuracy of the entire \ruleset{}. It represents the fraction of examples in the test dataset that the \ruleset{} correctly classifies.

\item \label{defn:Coverage-rate} \textit{Coverage-rate.} This metric measures the coverage of all the rules in the \ruleset{} relative to the total number of data points. $coverage\mbox{-}rate(S) = \frac{|\cup_{r \in S} cover(r)|}{N}$. We want coverage as large as possible.

\item \textit{Maximum rule-width.} This refers to the maximum width among all rules within the \ruleset{}. $Max.width(S) = max_{r\in S} width(r)$. A smaller value on this metric indicates more meaningful rules in the \ruleset{}.


\end{itemize}

\subsection{Comparison of Existing Benchmarks}
We conduct comprehensive experiments to assess the quality of \rulesets{} learned by the existing algorithms using ten real-world datasets from the UCI machine learning repository \cite{Dua2019} as well as two real-world Alzheimer’s disease datasets obtained from NACC Uniform Data Set \cite{Weintraub2009,nacc2007,f6d19c3fa05d461fbf210ecf49fc3c43}. Detailed experiment information is provided in the Experiments section (\ref{sec:expt_setup}), and the results are summarized in Table \ref{tab:summary_table}, \ref{tab:auc}, \ref{tab:coverage} and \ref{tab:rule_width_size}.

\begin{itemize}
\item 
 It's worth noting that \ids{} ~\cite{10.1145/2939672.2939874} and RIPPER \cite{10.5555/3091622.3091637} exhibit limitations. The mean coverage-rate over the test set is only 35\% for \ids{} and 31\% for RIPPER (see Table \ref{tab:coverage}). This suggests that the corresponding rules are ineffective for over 65\% of the examples, which is a significant concern for critical applications.
    
\item While DefragTrees ~\cite{pmlr-v84-hara18a} achieves high coverage, it faces challenges related to complex \rulesets{} with large rule-width (Table \ref{tab:rule_width_size}). Rules learned by DefragTrees have a rule-width greater than or equal to 8 in seven out of the twelve datasets, indicating increased model complexity.

\item In terms of Rule Set Accuracy (RSA), DefragTrees exhibits variations of more than 10\% compared to Random Forests (RF), which serves as an additional benchmark that, while highly accurate, lacks interpretability (Table 1).
\end{itemize}



These observations collectively highlight that existing state-of-the-art (SOTA) algorithms often do not produce high-quality \rulesets{} as measured by various metrics. Coverage is one of the most severely affected metrics. To this end, the first observation points to a significant challenge, which we refer to as the \textit{high accuracy-low coverage problem}.
\begin{center}
 \fbox{{\it Rule sets with high accuracy often have low coverage.}}  
\end{center}

The objective of this paper is to investigate algorithmic solutions to address this problem.

\section{\cdp{} : High Coverage Discriminative Parsimonious Rule Sets}
Motivated by the findings of the comparative study reported in the previous section, we introduce the problem of learning \cdp{}. \cdp{} is a \ruleset{} that aims to maximizes coverage while satisfying quality constraints related to \textit{rule-width}, \textit{rule-size}, \textit{rule-accuracy} and \textit{overlap} between rules as shown in Figure \ref{fig:drs_qualities}.

\begin{figure}[h]
\centering
\vspace{-.05in}
    \fbox{
    \begin{minipage}{25em}
$accuracy(r) \geq \alpha,  \qquad \qquad \forall r \in S, \alpha \in [0,1]$\\
$overlap\mbox{-} rate (r_{i}, r_{j}) \leq \beta,  \quad  \forall r_{i}, r_{j} \in S,   r_i\neq r_j, \beta \in [0,1]$\\
$rule \mbox{-}  size (S) \leq \gamma, \qquad \qquad   \gamma \geq 1$\\
$rule\mbox{-}width(r) \leq \delta, \qquad \quad  \forall r \in S, \delta \geq 1$ \\
\end{minipage}}
\caption{Quality constraints in \texttt{\cdp{}($\alpha, \beta, \gamma, \delta$)}}
\label{fig:drs_qualities}
\vspace{-.1in}
\end{figure} 

Throughout the subsequent discussion,, we omit the parameters $\alpha, \beta, \gamma$, and $\delta$ from \texttt{\chdr{}($\alpha, \beta, \gamma, \delta$)} for the sake of simplicity. 
The set of rules with high discriminative ability, referred to as \textit{discriminative candidate set} and  denoted by $B_{\alpha, \delta}$, is defined as follows:

\begin{equation*}
    B_{\alpha, \delta} = \{r |\ accuracy(r) \geq \alpha \ \&\  width(r) \leq \delta,  \forall r \in A \},
\end{equation*} 
Using this set we define the problem of \textit{high coverage, discriminative, parsimonious rule set learning} (\hdl{}) as follows:
\vspace{-.1cm}
\begin{align*}
& \underset{S \subseteq B_{\alpha, \delta}}{\text{maximize}}
\quad  \mathrm{coverage } (S), \; \text{s.t.}
\tag{\hdl{}} 
\label{problemq} \\
&  rule \mbox{-}  size (S) \leq \gamma,\;  overlap\mbox{-} rate (r_{i}, r_{j}) \leq \beta,  r_{i}, r_{j} \in S, \forall  i\neq j
\end{align*}

The optimization problem for learning \cdp{} is presented above.To the best of our knowledge, \cdp{} is a novel concept that has not been previously explored in the literature. \cdp{} ensures that we learn \ruleset{} that directly maximizes coverage without compromising accuracy and interpretability metrics.  The quality thresholds $\alpha, \beta, \gamma$, and $\delta$ are easily interpretable, making it straightforward to align them with specific business requirements, unlike approaches such as \ids{} which lack interpretability in their hyperparameters. It's worth noting that existing algorithms are unable to directly solve the \cdp{} problem due to the constraint on $overlap\mbox{-}rate$, which is not a matroid or knapsack constraint \cite{10.1137/090779346,10.1145/1536414.1536459}. This implies that custom solvers are needed to learn \cdp{}.


\section{Two Algorithms based on Sub-modular Maximization}
In this section, we present two approaches for learning \cdp{} as a cardinality-constrained sub-modular maximization problem: (1) \textit{Graph Rules Algorithm}~(\texttt{GRA}): This approach is designed to solve \hdl{} directly, and (2) \textit{ Greedy Algorithm}~(\texttt{GDY}): This method is employed to address a relaxed version of \hdl{} known as \textit{CDPRL with Mean Overlap (\shdl{})}, which is a sub-modular function with a cardinality constraint.


\subsection{GraphRules Algorithm (GRA)}
In this subsection, we introduce the \textit{GraphRules Algorithm} (\texttt{GRA}) for learning \cdp{}. We also establish the approximation guarantee of \texttt{GRA} and examine its computational complexity. 
\texttt{GRA} relies on a novel approach that transforms \hdl{} into a submodular optimization problem over a graph data structure. This transformation is achieved by constructing a discrete directed graph, known as \textit{overlap-graph} over the \textit{discriminative  candidate set $B_{\alpha, \delta}$}. \texttt{GRA} maps the rules in $B_{\alpha, \delta}$ as nodes of the graph and establishes edges in a manner that satisfies the overlap-rate constraints of \hdl{}. As a result, we frame \hdl{} as a sub-modular function with cardinality constraint over the \textit{overlap-graph}. \texttt{GRA},  as depicted in Algorithm \ref{alg:algorithm_gra}, mainly comprises two steps:
\begin{enumerate}

    \item \textit{Overlap-graph Construction:} \label{sec:overlap-graph} 
    In this step, we create an overlap graph $G = (V, E)$ where the vertices $V$ correspond to $B_{\alpha, \delta}$, and an edge $(r_{i}$, $r_{j}) \in E$ exists if and only if the $overlap\mbox{-} rate$ $(r_i, r_j)$ $ > \beta$, with $r_{i}, r_{j} \in V, i\neq j$. The vertices in the overlap graph represent the rules in the discriminative candidate set $B_{\alpha, \delta}$. A directed edge from one vertex $r_i$ to another vertex $r_j$ exists if more than $\beta$ percentage of data points covered by $r_i$ are also covered by $r_j$.   
    
    \item\textit{Find Rule-sets from Overlap-graph: } We begin with an empty rule set, denoted as $S$. In each iteration, we select the vertex with the highest coverage in the \textit{overlap-graph (G)} and add it to $S$. Subsequently, we rearrange $G$ by removing the selected vertex and any vertices connected by a directed edge to the selected vertex. If multiple vertices share the same highest coverage, we arbitrarily select one of the vertices with maximum coverage. Vertices are added to set $S$ until it contains a maximum of $\gamma$ rules, or the overlap-graph is empty.
\end{enumerate}

\begin{algorithm}[]
\caption{GraphRules}
\label{alg:algorithm_gra}
\begin{algorithmic} 
\STATE {\bfseries Input:}  Dataset $D = \{X,y\}_{i=1}^N$, set of rules in $B_{\alpha, \delta}$, parameters $\beta,\gamma$
\STATE {\bfseries Output:} \chdr{} $S$
\STATE $G \leftarrow (V,E)$: $V \leftarrow B_{\alpha, \delta}$, $E \leftarrow \{\}$ 
\FOR{each $r_i \in  B_{\alpha, \delta}$}   
\FOR{each $r_j \in  B_{\alpha, \delta}$}   
\IF{$overlap\mbox{-}rate(r_i, r_j) \geq \beta$}
    \STATE  $E \leftarrow E \cup (r_i, r_j)$  $ \qquad \  \triangleright$ overlap-graph construction
\ENDIF    
\ENDFOR
\ENDFOR
\STATE $S \leftarrow \{\}$
\WHILE{$G$ is non-empty or size($S) \leq \gamma$}
    \STATE $v \leftarrow $ max $\  coverage(v), \  \forall v \in V$
    \STATE  $S \leftarrow S \ \cup \  v$
    \FOR{each $r_i \in V\setminus v$}
        \IF{$(r_i, v) \in E$}
            \STATE $V  \leftarrow V \setminus \{r_i\}$
        \ENDIF
    \ENDFOR
    \STATE $V \leftarrow V \setminus {v}$   
\ENDWHILE
\STATE {\bfseries Return} S
\STATE
\end{algorithmic}
\end{algorithm}

In the context of multi-class classification, rules are generated for each class, and \texttt{GRA} is applied. Once the rule set is created, class labels are assigned as follows:
If a data point does not satisfy any rules in the \ruleset{}, it is assigned to the default rule, and it is the user's responsibility to determine the default rule.
If a data point satisfies exactly one rule, denoted as $r=(s,c)$, the data point is assigned to class $c$.
In cases where a data point satisfies more than one rule, the user has the discretion to determine the assignment criteria. One option is to assign the class corresponding to the rule with the highest coverage or highest accuracy on the training data, or to use a majority vote.

\paragraph{Relevant Definitions.} Given a set $X$, a discrete function $f:S\subseteq X$ is said to be - 
\textit{Submodular: } If $\forall C \subseteq D \subseteq X$ and $e \in (X-D)$, $f(C \cup e) - f(C) \geq f(D\cup e) - f(D)$.
\textit{Monotone: } If $\forall C\subseteq D \subseteq X, f(C) \leq f(D)$.

\subsection{Analysis of GraphRules Algorithm}
In this section, we establish the approximation guarantees for \texttt{GRA}. 
\begin{theorem}\label{gs_theorem}
Consider the problem \hdl{}, and let $S^{*}$ be its optimal solution.  Let $S = GraphRules(D, B_{\alpha, \delta}, \beta, \gamma)$ (Algorithm \eqref{alg:algorithm_gra}). Let $f(S)$ and $f(S^{*})$ are the coverage of $S$ and $S^*$, respectively. We have the following result: 
$f(S) \geq (1 - (1 - \frac{(1 - \beta)}{\gamma (1 + d\beta)})^\gamma) f(S^*)$, where 
$d = min_{u \in S} indegree(u)$.
\end{theorem}

\begin{proof}
Let $S^{*}$ be the optimal solution to the problem  \hdl{}. Let $S_{k}$ be the set of $k$ rules chosen  by algorithm till $k^{th}$ step. Similarly, assume that the optimal solution $S^{*}$ has $k$ rules given by $S^{*}_{k}$, where $k\leq \gamma$. Let $v_{k}$ be the rule added in $k^{th}$ iteration in $S_{k}$. Data-points that are not covered by $S^{*}_{k} - S_{k-1}$ are covered by some of the $k$ rules in $S^{*}_{k}$. One of the $v^{*}_{i} \in S^{*}_{k}$, $i=1,..., k$ must cover atleast $\frac{ f(S^{*}_{k}) - f(S_{k-1})}{k}$ data-points. 
Let $d$ be the  indegree of node $v_{k}$ (which corresponds to a rule) in overlap-graph $G$. 
Let $\beta$ is the overlap-rate. The datapoints that are covered by a rule $v_{k}$ might be covered by some other rules in $S_{k-1}$.  This means that the coverage of node $v_{k}$ is given as. $f(v_{k})  \geq \frac{ f(S^{*}_{k}) - f(S_{k-1})}{k(1 + d\beta)} $

From the above equation,  coverage of $v_{1}$ is given by $f(v_{1}) \geq \frac{f(S^{*}_{k})}{k(1 + d\beta)}  \geq \frac{(1-\beta)f(S^{*}_{k})}{k(1 + d\beta)}$ since $\beta \in (0,1)$.  By GraphRules algorithm $v_{1}$ is the node with highest coverage in overlap-graph. Node $v_{k}$ may have some datapoints overlapping with rules $v_{i} \in S_{k-1}$. The number of datapoints overlapping is less than $\beta f(v_{k})$, by virtue of our algorithm. This means that $f(S_{k-1} \cap v_{k}) \leq \beta f(v_{k})$. $- f(S_{k-1} \cap v_{k}) \geq - \beta f(v_{k})$.
We have the following:
\begin{equation}
\begin{split}
f(S_{k})  & =  f(S_{k-1})  +  f(v_{k}) - f(S_{k-1} \cap v_{k}) \\
& \geq f(S_{k-1})  + f(v_{k}) - \beta  f(v_{k})
\end{split}
\end{equation}
Substituting the value of $f(v_{k})$ in the above equation. 
\begin{equation}
f(S_{k})  \geq  f(S_{k-1})  +  (1-\beta) \frac{ (f(S^{*}_{k}) - f(S_{k-1})) }{k(1 + d\beta)} \\
\end{equation}
For brevity, we replace $\frac{1+d\beta}{1-\beta}$ by a constant $C$, we get, 
\begin{equation}
\begin{split}
f(S_{k})&\geq f(S_{k-1})  +   \frac{( f(S^{*}_{k}) - f(S_{k-1}))}{k C} \\
&\geq (1 - \frac{1}{k C})  f(S_{k-1}) + \frac{ f(S^{*}_{k})}{Ck} \\ 
\end{split}
\end{equation}
Since $\frac{f(S^{*}_{k})}{k} \leq \frac{f(S^{*}_{k-1})}{k-1}$, 
\begin{equation}
\begin{split}
f(S_{k})&\geq (1 - \frac{1}{k C})  f(S_{k-1}) + \frac{ f(S^{*}_{k})}{k C} \\ 
& \geq (1 - (1 - \frac{1}{k C})^k) f(S^{*}_{k})
\end{split}
\end{equation}
Since $k \leq \gamma$,  $(1 - (1 - \frac{1}{k C})^k) \geq (1 - (1 - \frac{1}{\gamma C})^\gamma)$. Therefore, $f(S) \geq (1 - (1 - \frac{1}{\gamma C})^\gamma) f(S^{*})$ or $f(S) \geq (1 - (1 - \frac{(1 - \beta)}{\gamma (1 + d\beta)})^\gamma) f(S^{*})$ 
\end{proof}


Now, let us consider the values of $\beta$ for which the overlap-graph denoted as $G=(V,E)$ is completely disconnected. This means that $overlap\mbox{-} rate(r_i, r_j) = 0$ for all $r_i, r_j \in V, i\neq j$. In this scenario, the optimization problem \hdl{} simplifies to $maximize_{S \subseteq B_{\alpha, \delta}, |S| \leq \gamma} coverage(S)$. This can be viewed as a monotone sub-modular function with cardinality constraint. For such cases, a greedy algorithm 
with a $(1 - e^{-1})$ -approximation for maximizing monotone submodular functions with a cardinality constraint is available, as described in \cite{Nemhauser1978}. 
Interestingly, this approach is equivalent to \texttt{GRA} when the \textit{overlap-graph} is completely disconnected. When the \textit{overlap-graph} is completely disconnected, \texttt{GRA} selects the vertex with the maximum coverage in the \textit{overlap-graph} and adds it to the \ruleset{}.

\begin{corollary}
Consider the problem \hdl{}, and let $S^{*}$ be its optimal solution.  Let $G = (V, E)$ denote the overlap-graph constructed from $B_{\alpha,\delta}$ and let $S = GraphRules(D, B_{\alpha, \delta}, \beta, \gamma)$ (Algorithm \ref{alg:algorithm_gra}). If $G$ is completely disconnected, i.e., $E=\phi$,  $f(S) \geq (1 - (1 - \frac{1}{\gamma})^\gamma) f(S^*)$, where $f(S)$ and $f(S^*)$ denotes the coverage of $S$ and $S^*$ respectively.
\end{corollary}

In the limiting case, when $\gamma \rightarrow \infty$, the approximation guarantee of \texttt{GRA} when the overlap-graph is completely disconnected is the same as the  greedy algorithm proposed in \cite{Nemhauser1978}.

\subsubsection{Computational Complexity of GRA}
Let $N$ be the number of records in the dataset and $A$ be the original rules obtained from the rule mining technique. The computational complexity of obtaining $B_{\alpha,\delta}$ is $\mathcal{O}(N|A|)$, where $|A|$ represents the number of rules in set $A$. 
The computational complexity of \texttt{GRA} is $\mathcal{O}(|B_{\alpha, \delta}|^{2}  )$, where $|B_{\alpha, \delta}|$ is the number of rules in set $B_{\alpha, \delta}$.
Notably as $\alpha$ increases, resulting in fewer rules in set $B_{\alpha, \delta}$, the time complexity decreases. Similarly, if $\delta$ decreases, leading to a reduced number of rules in $B_{\alpha, \delta}$, the time complexity also decreases.  Since \texttt{GRA} exhibits a running time that is quadratic in size of $B_{\alpha, \delta}$, in the next section, we will explore \texttt{GDY}, which has a running time that is nearly linear in the size of $B_{\alpha, \delta}$.

\subsection{Greedy Algorithm (GDY)}


In this section, we delve into the Greedy algorithm~\cite{NEURIPS2019_ad3019b8} designed for sub-modular maximization with a runtime that is nearly linear in the size of $B_{\alpha, \delta}$, to learn \cdp{}.  Since, the \textit{overlap-rate} constraint in \hdl{} does not fit the framework of a matroid or a knapsack constraint, we take a more relaxed approach to the problem.

In this relaxed problem, we aim to minimize \textit{mean-overlap} of \ruleset{} $\overline{S}$ while simultaneously maximizing coverage. The $mean\mbox{-}overlap(\overline{S})$  is defined as $mean\mbox{-}overlap(\overline{S}) = \frac{\sum_{r_i,r_j \in \overline{S}, i\neq j}overlap(r_i,r_j)}{|\overline{S}|^2}$. 
\begin{equation}
g(\overline{S}) = \mathrm{coverage}(\overline{S}) + \lambda \left(N - mean\mbox{-}overlap(\overline{S})\right)
\end{equation}

The relaxed problem referred to as \textit{High Coverage Discriminative Parsimonious Rule set Learning using Mean Overlap \shdl{}} is as follows:
\vspace{-.1cm}
\begin{equation}\tag{\shdl{}}\label{problemqvariant} 
\begin{aligned}
\underset{\overline{S} \subseteq B_{\alpha, \delta}, |\overline{S}| \leq \gamma}{\text{maximize}}
\ g(\overline{S})
\end{aligned}
\end{equation}
In \shdl{}, we optimize over the discriminative candidate set $B_{\alpha, \delta}$, ensuring that the quality constraints on \textit{accuracy} and \textit{rule-width} are satisfied. In \hdl{}, the $overlap\mbox{-}rate$ is strictly  $\le \beta$, while in \shdl{}, the \textit{mean-overlap} is upper bounded by $N$, where $N$ be the total number of data-points. 
We demonstrate that \shdl{} is a non-monotone sub-modular optimization function with a cardinality constraint. To solve this, we employ the Greedy algorithm designed for non-monotone submodular optimization functions with a cardinality constraint.
 

Let us consider the following proposition before proceeding with the solution to this problem. 

\begin{proposition}\label{propshdl}
Let $g(\overline{S}) = g_{1}(\overline{S}) + \lambda g_{2}(\overline{S})$ where, $g_{1}(\overline{S}) = coverage(\overline{S})$ and $g_{2}(\overline{S}) = N - \sum_{r_i, r_j \in \overline{S}, i\neq j} \frac{overlap(r_i, r_j)}{|\overline{S}|^2}$. For all $\lambda \ge 0, g(\overline{S})$ is a non-monotone sub-modular function.
\end{proposition}

\begin{proof}
Note that $g(\overline{S})$ is not monotone, if either $g_{1}$ or $g_{2}$ is non-monotone.
We can see that $g_{2}$ is non-monotone. Therefore, $g$ is non-monotone.

Suppose $C \subseteq D$ and $e \in X-D$. Let us assume that $g_{1}(C) = \sum_{r_{i} \in C} coverage(r_{i}) = x$ and $g_{1}(D) = \sum_{r_{i} \in D} coverage(r_{i}) = x + \epsilon$, where $\epsilon$ is the number of datapoints that are covered by rules in $D-C$ and not any rules in $C$.  Then,  $g_{1}(C \cup e) =  x + \epsilon'$, where $\epsilon'$ denote number of data-points that are covered by $e$ and not any rules in $C$. $g_{1}(D \cup e) =  x + \epsilon +( \epsilon' - \epsilon'')$, where $\epsilon''$ is the number of data-points that are covered by both $e$ and rules in $D-C$. 
\begin{equation}
\begin{split}
g_{1}(C \cup e) - g_{1}(C) & =  x + \epsilon' - x = \epsilon'\\
g_{1}(D \cup e) - g_{1}(D) & =   x + \epsilon +( \epsilon' - \epsilon'') - ( x + \epsilon) 
= \epsilon' - \epsilon'' \\
\implies  g_{1}(C \cup e) - g_{1}(C)  &\geq g_{1}(D \cup e) - g_{1}(D)\\
\end{split}
\end{equation}
Hence, $g_{1}$ is sub-modular. A similar argument demonstrates that $g_{2}$ is sub-modular. Therefore  $g$ is sub-modular.
\end{proof}

We employ the Greedy algorithm presented in ~\cite{NEURIPS2019_ad3019b8} to compute the solution for \shdl{}. The Greedy algorithm from ~\cite{NEURIPS2019_ad3019b8} is an efficient, nearly linear time, deterministic approximation algorithm designed for optimizing non-monotone sub-modular functions with a cardinality constraint. The following theorem provides an approximation guarantee for \texttt{GDY}:

\begin{theorem}\label{claim:sddp}
Let $g$ be a non-monotone sub-modular function with cardinality constraint.
Let $\overline{S^*} = \text{argmax}_{\overline{S} \subseteq B_{\alpha, \delta}, |\overline{S}| \leq \gamma} g(\overline{S})$.
Let $\overline{S} = Greedy(g, B_{\alpha, \delta}, \rho, \gamma)$. Then,  $g(\overline{S}) \geq \frac{1}{4} g(\overline{S^*})$ and Greedy makes $\mathcal{O}(|B_{\alpha, \delta}| log(\gamma))$ queries to $f$, where $|B_{\alpha, \delta}|$ is the number of rules in $B_{\alpha, \delta}$.
\end{theorem}
\begin{proof}
Kuhnle et al.\cite{NEURIPS2019_ad3019b8}\label{als_theorem} states the following theorem:
 \textit{Let $g:2^{[n]}\rightarrow \mathbb{R}^{+}$ be  sub-modular, let $k \in [n]$ and $\epsilon > 0$. Let $O = argmax_{\overline{S} \leq k} g(\overline{S})$. Choose $\rho$ such that $(1-6\rho)/4 > 1/4-\epsilon$ and let $C = FIG(g, k, \rho)$. Then, $g(C) \geq (1-6\rho) g(O)/4 \geq (1/4-\epsilon)g(O)$ and FIG algorithm makes $\mathcal{O}(\frac{n}{\rho}log\frac{k}{\rho})$ queries to $f$.} This theorem applies since the objective g is sub-modular with cardinality constraint as shown in Proposition~1. 
\end{proof}

\section{EXPERIMENTS}
In this section, we assess the performance of \cdp{} learned using \texttt{GRA} and \texttt{GDY}. We provide an overview of the baselines and experimental setup. Our analysis encompasses the classification performance and the quality constraints of rule sets learned from ten real-world datasets from the UCI machine learning repository \cite{Dua2019} as well as two real-world Alzheimer’s disease datasets obtained from the NACC Uniform Data Set \cite{Weintraub2009,nacc2007,f6d19c3fa05d461fbf210ecf49fc3c43}.

\subsection{Baselines and Experimental Setup}\label{sec:expt_setup}

\begin{table*}[t]
\centering
\begin{tabular}{lrrccccccc}
\hline
Datasets & \#records & \#rules & GRA &  GDY & IDS  & Defrag & RIPPER & LR & RF \\
\hline
Heart & 303 & 1328 & \textbf{0.86}  & 0.85 & 0.80 & 0.71 & 0.79 & 0.84 & 0.85 \\
Student & 649 & 1413 & \textbf{0.71} & 0.70 & 0.70 & 0.59 & 0.63  & 0.69 & 0.70 \\
Breast & 699 & 1254 & \textbf{0.95}  & 0.93  & 0.92 & 0.91 & 0.92 & 0.93 & \textbf{0.95} \\
Thyroid & 3428 & 1043 & \textbf{0.99} & \textbf{0.99} & 0.98 & 0.97 & 0.98 & 0.93 & \textbf{0.99} \\
Abalone & 4174 & 1528 & \textbf{0.84} & 0.83 & 0.83 & 0.81 & 0.78 & 0.80 & \textbf{0.84} \\
Spam & 4601 & 1252 & \textbf{0.89} & 0.88 & 0.88 & 0.73 & 0.88 & \textbf{0.91} & \textbf{0.91} \\
Wine & 4899 & 1484 & \textbf{0.75} & 0.72 & 0.73 & 0.67 & 0.67 & 0.73 & \textbf{0.77} \\
Magic & 19020 & 1530 & \textbf{0.82}  & \textbf{0.82}  & 0.80 & 0.66 & 0.66 & 0.79 & \textbf{0.84} \\
Adult & 32561 & 160 & \textbf{0.82} & 0.81 & 0.81 & 0.81 & 0.81 & 0.79 & \textbf{0.85}\\
Bank & 41188 & 1503 & \textbf{0.90} & \textbf{0.90}  & 0.89 & 0.89 & 0.89 & 0.77 & \textbf{0.90} \\
MCI Sc & 108399 & 397 & \textbf{0.74} &  \textbf{0.74} & 0.72 & 0.70 & 0.73 & 0.72 & \textbf{0.74} \\
Dementia Sc & 145028 & 318 & \textbf{0.87} & 0.86 & 0.82 &  0.66 & 0.85 & 0.83 & \textbf{0.87} \\
\hline
\end{tabular}
\caption{Rule Set Accuracy(RSA) of rule sets  generated by different algorithms on test data-set. From the Table, it is evident that GRA has the highest accuracy when compared with rule set based baselines}
\label{tab:auc}
\vspace{-0.25cm}
\end{table*}

We compare the \cdp{} learned using \texttt{GRA} and \texttt{GDY} with \rulesets{} learned using state-of-the-art algorithms : \texttt{IDS} \footnote{https://github.com/lvhimabindu/interpretable\_decision\_sets} \cite{10.1145/2939672.2939874}, DefragTrees \footnote{https://github.com/sato9hara/defragTrees} \texttt{(DT)} \cite{pmlr-v84-hara18a}, and \texttt{RIPPER} \footnote{https://github.com/imoscovitz/wittgenstein} \cite{10.5555/3091622.3091637}. Additionally, we compare them with standard models such as logistic regression (LR) and random forest (RF) \footnote{We used scikit learn implementation \cite{scikit-learn} of logistic regression (LR) and random forest (RF).} to demonstrate the competitiveness of rule-set models.

\paragraph{Initial rules.} \texttt{GRA}, \texttt{GDY}, and \texttt{IDS} require a pre-mined set of rules as input. These rules are pre-mined from gradient-boosting decision tree ensembles that are fitted to the training data \cite{friedman2005predictive}. Each rule is derived from the decision tree ensembles by concatenating condition-value pairs from the root node to the leaf node, with the leaf node specifying the label. In our experiments, we employed the same set of rules as input for \texttt{GRA}, \texttt{GDY}, and \texttt{IDS}. DefragTrees  \cite{pmlr-v84-hara18a} and \texttt{RIPPER} \cite{10.5555/3091622.3091637} do not require the pre-mining rules. If a datapoint does not satisfy any of the rules in a \ruleset{}, it is assigned a default rule.

\paragraph{Hyperparameters.} For \texttt{GRA} and \texttt{GDY}, we sample $\alpha \in [0.60, 0.95], \beta \in [0.40, 0.80], \gamma \in [3, 12], \delta \in [3, 5], \lambda \in [1, 500]$. For \texttt{IDS}, we choose $\lambda_{i} \in [1, 500], i=[1,7]$. For \texttt{DefragTrees }, we chose the hyper-parameter $K_{max} \in [3,10]$. For \texttt{RIPPER}, we set the maximum-rules to $10$. For random forest (RF), we chose the number of estimators in the range [10, 500] and max-depth in the range [3, 7]. For logistic regression (LR), we considered l2 regularization with hyper-parameter $C$ in the range [0.01, 100].

\subsection{Numerical Results }\label{sec:numeric_res}

\subsubsection{Discriminative Power}

The disciminative power of \rulesets{} is quantified by the  \textit{Rule Set Accuracy (RSA)} metric, as defined in Section 3.1. The empirical evaluation reveals that \texttt{GRA} exhibits significantly higher discriminative power when compared to \texttt{GDY} and other baseline methods. \texttt{GRA} achieves an average RSA of 0.85, while \texttt{GDY} attains an average RSA of 0.84 (Table \ref{tab:summary_table}). This is notably higher than the average RSA values for \texttt{IDS}, Defrag and \texttt{RIPPER}, as presented in Table \ref{tab:summary_table}. The RSA for each dataset, along with dataset size and rule-count, can be found in Table \ref{tab:auc}. 

The results demonstrate that the \ruleset{} learned by \texttt{GRA} either matches or surpasses the performance of Random Forest(\texttt{RF}) in $7$ out of $12$ datasets and falls within a 3\% margin of \texttt{RF}'s RSA in the remaining five datasets.  \texttt{GRA} consistently outperforms other interpretable models, including  \texttt{IDS}, \texttt{RIPPER}, and DefragTrees. While \texttt{GDY}, achieves comparable performance to \texttt{RF} in 5 out of 12 datasets, it performs similar to or better than \texttt{IDS}, which is the next best algorithm. This clearly illustrates that \texttt{GRA} and \texttt{GDY} outperforms other interpretable models while achieving performance comparable to that of complex models like \texttt{RF}. 


\begin{table}[htbp]
\centering
\begin{tabular}{lccccc}
\hline
Datasets & GRA &  GDY & IDS & Defrag & RIPPER\\
\hline
Heart & \textbf{0.99} & 0.87 & 0.46 & 0.89 & 0.49 \\
Student &  \textbf{0.97}  & 0.94 & 0.22 & 0.90 & 0.06 \\
Breast & \textbf{1.00} &  0.99 & 0.38 & 0.95 & 0.39  \\
Thyroid & \textbf{0.99}  & 0.97 & 0.16 & 0.97 & 0.09\\
Abalone & \textbf{0.94} & 0.91 & 0.30 & 0.84 & 0.43  \\
Spam &  \textbf{0.96}  & 0.90 & 0.76 & 0.95 & 0.39\\
Wine & \textbf{0.95} & 0.71 & 0.31 & 0.83 & 0.48 \\
Magic & \textbf{0.99} & \textbf{0.99} & 0.21 & 0.93 & 0.40\\
Adult & \textbf{0.90} & 0.85 & 0.06 & 0.87 & 0.33 \\
Bank & \textbf{0.99} & \textbf{0.99} & 0.10 & 0.91 & 0.10\\
MCI Sc & \textbf{0.97} & 0.95  & 0.42 & 0.81 & 0.31\\
Dem Sc & \textbf{0.98} & 0.97 & 0.89 & 0.92 & 0.32 \\ 
\hline
\end{tabular}
\caption{Coverage-rate of \rulesets{} on test data-set. From Table, it is evident that GRA and GDY has atleast $20\%$ increase in coverage-rate compared to \texttt{IDS} and \texttt{RIPPER}.}
\label{tab:coverage}
\end{table}
\begin{table}
\centering
\begin{tabular}{lccccc}
\hline
Datasets & GRA &  GDY & IDS & Defrag & RIPPER  \\
\hline
Heart &  4\&7 & 4\&7 & 4\&7 & \textbf{4}\&\textbf{3} & 3\&6 \\

Student & 4\&8 & 4\&7 & 4\&8 & 5\&\textbf{5} & 5\&3\\

Breast & 4\&5 & \textbf{4}\&\textbf{3} & 4\&5 & 4\&4 & 5\&7   \\

Thyroid & \textbf{4}\&\textbf{3} &  \textbf{4}\&\textbf{3}  & 4\&4 & \textbf{4}\&\textbf{3} & 4\&3\\

Abalone & \textbf{4}\&\textbf{7} & \textbf{4}\&\textbf{7}  & \textbf{4}\&\textbf{7}  & 8\&6 & 4\&10 \\

Spam & \textbf{4}\&\textbf{6}  & \textbf{4}\&\textbf{6} & 4\&9 & 23\&7 & 4\&8\\

Wine & \textbf{4}\&\textbf{6} &  \textbf{4}\&\textbf{6} & \textbf{4}\&\textbf{6} & 10\&7 &  4\&9 \\

Magic & \textbf{4}\&\textbf{10} & 4\&12 & \textbf{4}\&\textbf{10} & 11\&7 & 5\&7\\

Adult & \textbf{4}\&\textbf{9} & \textbf{4}\&\textbf{9} & \textbf{4}\&\textbf{9} & 9\&9 & 5\&6  \\

Bank & \textbf{4}\&\textbf{8}  & \textbf{4}\&\textbf{8} & \textbf{4}\&\textbf{8} & 13\&9 & 5\&10 \\

MCI Sc & \textbf{4}\&\textbf{9}  & \textbf{4}\&\textbf{9}  & \textbf{4}\&\textbf{9} & 7\&9 & \textbf{4}\&6\\

Dem Sc & \textbf{4}\&\textbf{11} & \textbf{4}\&\textbf{11} &  \textbf{4}\&\textbf{11} & 9\&8 & 7\&5 \\
\hline
\end{tabular}
\caption{Maximum Rule-Width\&Rule size of \rulesets{} on test data-set. From Table, it is evident that \cdp{} has low value of maximum rule-width or low rule count or both.}
\label{tab:rule_width_size}
\end{table}
 
\subsubsection{Interpretability}

Inteepretability is determined by coverage and parsimony of \rulesets{}. As defined in Section 3.1, we have established metrics for evaluating the interpretability of \rulesets{}, including \textit{coverage-rate, maximum rule-width} and \textit{rule-size}. Table \ref{tab:coverage} provides the coverage-rate of our algorithms and baselines for different datasets. The results indicate that \texttt{GRA} exhibits a higher coverage-rate compared to state-of-the-art baselines. \texttt{GRA} consistently achieves a coverage-rate of more than 90\% across all datasets, with an average coverage-rate of 97\%. In contrast, \texttt{GDY} achieves an average coverage-rate of 91\% (Table \ref{tab:summary_table}). \texttt{IDS}, which ranks as the next best baseline in terms of discriminative power, lags behind with an average coverage-rate of 35\% (Table \ref{tab:summary_table}).

The parsimony of \rulesets{} is gauged by maximum rule-width and rule-size, as outlined in Table \ref{tab:rule_width_size}. When examining the maximum rule-width and rule-size, it becomes apparent that \texttt{GRA} and \texttt{GDY}'s values are comparable to those of \texttt{IDS} and \texttt{RIPPER}.  
These algorithms have substantially lower maximum rule widths when compared to DefragTrees from Table \ref{tab:rule_width_size}. On average, the maximum rule-width of DefragTrees is more than twice that of  \texttt{GRA} and \texttt{GDY} (Table \ref{tab:summary_table}), negatively impacting interpretability.

\begin{table}[]
    \centering
    \begin{tabular}{lccccc}
    \hline
Metric & GRA & GDY & IDS & Defrag & RIPPER \\
    \hline
RSA &  \textbf{0.85} & 0.84 & 0.82 & 0.76 & 0.80 \\
Coverage & \textbf{0.97} & 0.91 & 0.35 & 0.90 & 0.31\\
Rule-width & \textbf{4} & \textbf{4} & \textbf{4} & 10 & 5\\
Rule-size & 7 & 7 & 8 & \textbf{6} & 7\\
\hline 
    \end{tabular}
    \caption{Summary of the results obtained by averaging across the twelve datasets : RSA (Table \ref{tab:auc}), Coverage-rate (Table \ref{tab:coverage}), Maximum Rule Width and Rule size (Table \ref{tab:rule_width_size}). \texttt{GRA} performs better than baselines in most of the metrics.}
    \label{tab:summary_table}
\end{table}

\subsection{Case Study : Designing Neurocognitive Tests through Ruleset Learning}
In this case study, we attempt to design Neurocognitive Tests as an application of \ruleset{} learning. Designing medical diagnostic tests using \rulesets{} has not been studied in the literature. Manual design of tests is time intensive and may fail to be fully exhaustive, potentially excluding important relationships. There are some previous attempts \cite{brarati} to reduce the overall number of Neurocognitive tests. 
It is important to design tests to diagnose various diseases such as Alzheimer’s Disease. Alzheimer’s Disease is characterized by initial memory impairment and cognitive decline and affects over 30 million people worldwide, with a high social burden for patients and caregivers. Neurocognitive tests are regarded as the most accurate indication of early impairment.   

Neurocognitive tests comprise a set of tasks performed by a subject, designed to assess the subject's state in the cognitive domains of orientation, memory, attention, visuospatial, language, and executive function.  Various attributes (such as time taken, error rate, etc.) are extracted for each task. Each attribute corresponds primarily to a single cognitive domain and is considered a feature in the classifier. 
A collection of features, together with its limiting values, corresponds to a rule. The duration of the testing is directly proportional to the number of tests conducted on the patient and the number of attributes.  

Since Neurocognitive tests are long and tedious, having a rule-based system is helpful, where the entire battery of tests need not be administered to every subject. Hence, to reduce the testing duration while improving its effectiveness, it is essential to obtain the minimum set of rules covering all the cognitive domains yielding as high an accuracy as possible. We believe that \cdp{} learning is a good fit for designing accurate and interpretable diagnostis tests and can generate efficient Neurocognitive tests in our case study.

\paragraph{Dataset and Methods} The study population comprised 145k  subjects   from the NACC Uniform Data Set \cite{Weintraub2009,nacc2007,f6d19c3fa05d461fbf210ecf49fc3c43} from approximately 32 Alzheimer’s Disease Research Centers (ADRCs) between 2005 and 2019. Each record consists of 69 attributes from 11 chosen standardized neuropsychological tests. 
The details of tests are in supplementary material. Clinical assessments grade the disease into categories measured by Clinical Dementia Rating (CDR®), which ranges from a value of 0 for cognitively normal individuals, 0.5 for Mild Cognitive Impairment (MCI), a prodrome of Alzheimer’s to 1, and higher values which correspond to  Alzheimer’s Disease of various grades.

We have proposed two classifiers in tandem for diagnosing Alzheimer’s Disease -  \textit{Dementia Screener (Dem.Sc.)} and \textit{MCI Screener (MCI.Sc.)}. 
\textit{Dementia Screener} segregates a subject as definitively suffering from Dementia (CDR = 1, 2, 3) or not (CDR = 0, 0.5). The \textit{MCI Screener} distinguishes between subjects with CDR = 0 and CDR = 0.5 (Normal vs. Mild Cognitive Impairment (MCI)). In subjects suspected of cognitive impairment, the state of Mild Cognitive Impairment (CDR = 0.5) is the most difficult to classify \cite{brarati}, since they might fall into either of CDR = 0 or CDR = 1 category. Hence to diagnose Alzheimer’s Disease, we first pass a subject through the Dementia Screener to screen if the subject is definitively suffering from Dementia or not. If the subject is not definitively suffering from Dementia, we pass the subject through MCI Screener to diagnose between Normal and MCI. The clinical interventions for the Dementia and MCI cases could be different.

\begin{table*}[t]
\centering
\begin{tabular}{lrrlllllll}
\hline
Datasets & \#records & \#rules & GRA &  GDY & IDS  & DT & LIBRE & LR & RF \\
\hline
MCI Sc & 108399 & 397 & \textbf{0.74} &  \textbf{0.74} & 0.72 & 0.70 & 0.51 & 0.72 & \textbf{0.74} \\
Dementia Sc & 145028 & 318 & \textbf{0.87} & 0.86 & 0.82 &  0.66 & 0.51 & 0.83 & \textbf{0.87} \\
\hline
\end{tabular}
\caption{Rule Set Accuracy(RSA) of rule sets  generated by different algorithms on Alzheimer’s Disease test data-set. From the Table, it is evident that GRA has the highest accuracy when compared with rule set based baselines}
\label{tab:ad_auc}
\vspace{-0.25cm}
\end{table*} 



\paragraph{Results and Discussion}
Table \ref{tab:ad_auc} reports RSA of \rulesets{} learned by \texttt{GRA}, \texttt{GDY}, and baselines for Dementia Screener and MCI Screener. The RSA of the \ruleset{} learned by \texttt{GRA} is higher than that baselines and is the same as that of RF (Table \ref{tab:ad_auc}). 
We consider \rulesets{} learned by algorithms for Dementia and MCI Screener for which the coverage-rate is atleast 70\% . Although \texttt{IDS} has reasonable accuracy, the coverage-rate is low. Among the algorithms considered, \texttt{GRA}, \texttt{GDY} and \texttt{DT} has more than 70\% coverage-rate. 
From the results, \texttt{GRA} and \texttt{GDY} exhibits atleast 18\% reduction in number of features compared to \texttt{DT}. Though \texttt{GDY} is not as accurate compared to \texttt{GRA}, it is a good candidate for designing Neurocognitive tests. \texttt{GDY} has time complexity nearly linear in number of rules in $B_{\alpha,\delta}$.

It is evident from the results that \texttt{GRA} learn \ruleset{} for Dementia Screener and MCI Screener with high coverage-rate, discriminative power and parsimony. To our best knowledge, this work is the first attempt to make use of the fact that early onset of Alzheimer's (consisting of the transition between Normal and MCI individuals) follows different rules from later progression (CDR = 1 and onwards) to arrive at a reduced number of tests. Dementia Screener can help quickly detect cases that are not yet fully afflicted with the disease. The MCI Screener can focus on the rules for people not yet afflicted with the disease, thereby reducing the overall testing time. This will reduce the cognitive stress of the individuals undergoing the cognitive tests for Alzheimer's disease and make the screening process efficient. Clinical validation of the \rulesets{} will be seperately conducted.


\section{CONCLUSION}
We believe that developing mechanisms to learn accurate and interpretable rule-based models will improve the adoption of machine learning models in critical applications, where the model is liable to justify the decision to the stakeholders. 
This paper presents a novel approach for learning \rulesets{} that are accurate and interpretable by proposing   \cdp{}. We have developed two algorithms for learning \cdp{}, the \texttt{GRA}, and the \texttt{GDY}, and obtained theoretical guarantees for both approaches. The empirical results show that our techniques perform better than existing rule learning algorithms. There are multiple exciting lines of future work to explore. It would be interesting to develop an objective function to learn CDPRL, where it is possible to relax at least one of the constraints. Consequently, it will be interesting to develop a soft-version of overlap-graph for the GRA.


%
%
%
\bibliographystyle{splncs04}
\bibliography{references}





\end{document}